\documentclass{article} 
\usepackage{icomp2026_conference,times}

\usepackage{amsmath,amsfonts,bm}

\def\Secref#1{Section~\ref{#1}}

\def\eqref#1{~\ref{#1}}
\def\Eqref#1{Equation~\ref{#1}}

\def\1{\bm{1}}

\def\vzero{{\bm{0}}}

\def\vg{{\bm{g}}}
\def\vh{{\bm{h}}}

\def\vu{{\bm{u}}}
\def\vv{{\bm{v}}}

\def\vx{{\bm{x}}}

\def\mC{{\bm{C}}}
\def\mD{{\bm{D}}}

\def\mF{{\bm{F}}}

\def\mI{{\bm{I}}}

\def\mK{{\bm{K}}}

\def\mM{{\bm{M}}}

\def\mP{{\bm{P}}}

\def\mR{{\bm{R}}}
\def\mS{{\bm{S}}}

\def\mU{{\bm{U}}}
\def\mV{{\bm{V}}}

\def\mX{{\bm{X}}}
\def\mY{{\bm{Y}}}
\def\mZ{{\bm{Z}}}

\def\mSigma{{\bm{\Sigma}}}

\DeclareMathAlphabet{\mathsfit}{\encodingdefault}{\sfdefault}{m}{sl}
\SetMathAlphabet{\mathsfit}{bold}{\encodingdefault}{\sfdefault}{bx}{n}

\newcommand{\E}{\mathbb{E}}
\newcommand{\Ls}{\mathcal{L}}
\newcommand{\R}{\mathbb{R}}

\newcommand{\KL}{D_{\mathrm{KL}}}

\DeclareMathOperator*{\argmin}{arg\,min}

\newcommand{\Linner}{\Ls_{\mathrm{inner}}}

\newcommand{\Louter}{\Ls_{\mathrm{outer}}}
\newcommand{\Loutertilde}{\widetilde{\Ls}_{\mathrm{outer}}}

\newcommand{\Fisher}{\mF}

\usepackage{algorithm}
\usepackage{algpseudocode}
\usepackage{caption}
\usepackage{amssymb}

\algrenewcommand\algorithmicrequire{\textbf{Input:}}
\algrenewcommand\algorithmicensure{\textbf{Output:}}

\newcommand{\HypergradientDescentWithExplicitFisher}{

\begin{algorithm}[!htbp]
\footnotesize
\caption{Hypergradient Descent with Explicit Fisher}
\label{alg:explicitfisherhd}

\begin{algorithmic}[1]
\Require expert dataset $\mathcal D_{\mathrm{expert}}$, iterations $K$, damping $\lambda$
\State Initialize $\theta,\phi$
\For{$k=1,\ldots,K$}

\State
$\displaystyle
\theta^\star(\phi)
\approx
\argmin_{\vartheta}
\Linner(\vartheta,\phi)$
\Comment{e.g., PPO, SAC, or REINFORCE}

\State $\theta \gets \theta^\star(\phi)$

\State Collect trajectories
$\mathcal D_{\mathrm{agent}} \sim p_{\pi_{\theta}}$

\State Compute $\widehat{\vg}$
\Comment{\Eqref{eq:method-outer-gradient}}

\State Compute $\widehat{\mF}$
\Comment{\Eqref{eq:method-empirical-fisher}}

\State Solve for $\widehat{\vv}$
\Comment{\Eqref{eq:method-fisher-system}}

\State Compute $\widehat{\vu}$
\Comment{\Eqref{eq:method-cross-vector}}

\State
$\widehat{\vh}\gets-\widehat{\vu}$

\State Clip $\widehat{\vh}$ and update $\phi \gets \phi - \eta_\phi\widehat{\vh}$

\EndFor
\end{algorithmic}
\end{algorithm}

}

\newcommand{\SCFD}{%

\begin{algorithm}[!htbp]
\caption{Spectral Compensation Frequent Directions Solver}
\label{alg:scfd}

\begin{algorithmic}[1]

\Require sketch size $m$, dimension $d$, regularization $\lambda>0$

\State $\mS\gets\vzero_{m\times d}$
\Comment{compressed sketch}

\State $\mR\gets\varnothing$
\Comment{pending rows}

\State $\alpha\gets\lambda$
\State $\mD\gets\lambda^{-1}\mI_m$

\Statex

\Procedure{Append}{$\vx$}
    \State
    $\displaystyle
    \mR
    \gets
    \begin{bmatrix}
    \mR\\
    \vx^\top
    \end{bmatrix}$

    \If{$|\mR|=m$}
        \State Form the matrix $\displaystyle
        \mY\gets
        \begin{bmatrix}
            \mS\\
            \mR
        \end{bmatrix}
        \in\R^{2m\times d}$

        \State Compute SVD: $\left[\mU,\mSigma,\mV\right] \gets \operatorname{SVD}(\mY)$

        \State $\delta\gets\sigma_m^2$
        \State $\alpha\gets\alpha+\delta$

        \For{$i=1,\ldots,m$}
            \State
            $\displaystyle
            \bar{\sigma}_i
            \gets
            \sqrt{\max\{\sigma_i^2-\delta,0\}}
            $
        \EndFor

        \State Update the compressed sketch
        $\displaystyle
        \mS
        \gets
        \operatorname{diag}
        (\bar{\sigma}_1,\ldots,\bar{\sigma}_m)
        \mV_{1:m}^\top
        $

        \State Update the diagonal inverse statistics
        $\displaystyle
        \mD
        \gets
        \operatorname{diag}
        \left(
            \frac{1}{\bar{\sigma}_1^2+\alpha},
            \ldots,
            \frac{1}{\bar{\sigma}_m^2+\alpha}
        \right)
        $

        \State $\mR\gets\varnothing$
    \EndIf
\EndProcedure

\Statex

\Procedure{Solve}{$\vg$}
    \If{$\mR=\varnothing$}
        \State $\mZ\gets\mS$
        \State $\mM\gets\mD$
    \Else
        \State $\mC\gets\mS\mR^\top$
        \State $\mP\gets\mD\mC$

        \State Compute the Schur complement
        $\displaystyle
        \mK
        \gets
        \mR\mR^\top
        -
        \mC^\top\mP
        +
        \alpha\mI_{|\mR|}
        $

        \State Compute $\mK^{-1}$ using a Cholesky factorization

        \State Form the inverse row-space matrix
        $\displaystyle
        \mM
        \gets
        \begin{bmatrix}
            \mD+\mP\mK^{-1}\mP^\top
            &
            -\mP\mK^{-1}
            \\[2pt]
            -\mK^{-1}\mP^\top
            &
            \mK^{-1}
        \end{bmatrix}
        $

        \State Form the current sketch matrix
        $\displaystyle
        \mZ
        \gets
        \begin{bmatrix}
            \mS\\
            \mR
        \end{bmatrix}
        $
    \EndIf

    \State Compute the inverse-vector product
    \Statex \hspace{\algorithmicindent}
    $\displaystyle
    \vv
    \gets
    \alpha^{-1}
    \left(
        \vg-\mZ^\top\mM\mZ\vg
    \right)
    $

    \State \Return $\vv$
\EndProcedure

\end{algorithmic}
\end{algorithm}
}

\newcommand{\HypergradientDescentWithFisherSketching}{%

\begin{algorithm}[!htbp]
\footnotesize
\caption{Hypergradient Descent with Fisher Sketching}
\label{alg:sketched-fishernhd}

\begin{algorithmic}[1]
\Require expert dataset $\mathcal D_{\mathrm{expert}}$,
iterations $K$, sketch size $m$, damping $\lambda$

\State Initialize $\theta,\phi$

\For{$k=1,\ldots,K$}

    \State
    $\displaystyle
    \theta^\star(\phi)
    \approx
    \argmin_{\vartheta}
    \Linner(\vartheta,\phi)$
    \Comment{e.g., PPO, SAC, or REINFORCE}

    \State $\theta\gets\theta^\star(\phi)$

    \State Collect trajectories
    $\mathcal D_{\mathrm{agent}}
    \sim p_{\pi_\theta}$

    \State Compute $\widehat{\vg}$
    \Comment{\Eqref{eq:method-outer-gradient}}

    \State Initialize a SCFD sketch
    \Comment{\Cref{alg:scfd}}

    \ForAll{$\tau\in\mathcal D_{\mathrm{agent}}$}
        \For{$t=1,\ldots,|\tau|$}
            \State
            $\displaystyle
            \vx
            \gets
            \sqrt{
                \frac{\alpha}{N_{\text{agent}}}
                \gamma^{t-1}
            }
            \nabla_\theta
            \log\pi_\theta(a_t\mid s_t)$

            \State \Call{Append}{$\vx$}
        \EndFor
    \EndFor

    \State
    $\widehat{\vv}
    \gets
    \Call{Solve}{\widehat{\vg}}$

    \State Compute $\widehat{\vu}$
    \Comment{\Eqref{eq:method-cross-vector}}

    \State
    $\widehat{\vh}\gets-\widehat{\vu}$

    \State Clip $\widehat{\vh}$ and update $\phi \gets \phi - \eta_\phi\widehat{\vh}$

\EndFor
\end{algorithmic}
\end{algorithm}
}

\icompfinalcopy

\usepackage{url}

\usepackage{amsthm}
\usepackage{mathtools}
\usepackage{booktabs}
\usepackage{multirow}

\usepackage{xcolor}
\usepackage[textwidth=3cm,textsize=footnotesize]{todonotes}
\usepackage{hyperref}
\usepackage{cleveref}

\theoremstyle{plain}

\newtheorem{proposition}{Proposition}[section]
\newtheorem{lemma}{Lemma}[section]
\newtheorem{corollary}{Corollary}[section]

\theoremstyle{definition}

\theoremstyle{remark}

\renewcommand{\Pr}{\mathbb{P}}
\newcommand{\Loss}{\mathcal{L}}
\newcommand{\mcS}{\mathcal{S}}
\newcommand{\mcA}{\mathcal{A}}
\newcommand{\mcP}{\mathcal{P}}

\newcommand{\eqsp}{\;}
\newcommand{\wt}[1]{\widetilde{#1}}
\newcommand{\rset}{\mathbb{R}}

\title{Efficient Hypergradient Descent for~Inverse~Reinforcement Learning}

\author{Nikita Sevriukov, Anna Barabanova, Uliana Gagarina, Karina Ivanova
\& Sofiia Kasaeva  \\
HSE University \\
Moscow, Russia \\
\texttt{\{adbarabanova,uigagarina,kaalivanova,sakasaeva,nvsevriukov\}@edu.hse.ru} \\
\AND
Ilya Levin \& Marina Sheshukova \\
HSE University \\
Moscow, Russia \\
\texttt{\{ivlevin,msheshukova\}@hse.ru}
}

\begin{document}

\maketitle

\begin{abstract}
Inverse reinforcement learning (IRL) aims to recover a reward function under which the resulting policy reproduces the behavior observed in expert demonstrations. A natural approach is to formulate IRL as a bilevel optimization problem, in which the inner level corresponds to policy optimization under the learned reward and the outer level measures the discrepancy between the induced policy and expert data. However, this formulation is computationally challenging in practice because the outer update requires a hypergradient involving an inverse-Hessian-vector product for the inner objective. We address this challenge by showing that, at the inner optimum, the Hessian of the inner objective is proportional to the Fisher information matrix of the policy, yielding a structured Fisher-based hypergradient closely related to Natural Hypergradient Descent. To address the resulting scalability bottleneck associated with large Fisher matrices, we approximate the required inverse-Fisher-vector product using a streaming spectral sketch, avoiding explicit construction of the Fisher matrix. We evaluate our approach against a first-order stochastic bilevel baseline across discrete- and continuous-control environments. The results demonstrate competitive policy performance and strong reward-ranking quality, while Fisher sketching reduces curvature-storage complexity and can improve computational efficiency relative to an explicit Fisher solver.
\end{abstract}

\section{Introduction}

Inverse reinforcement learning (IRL) seeks to recover a reward function from expert demonstrations. Unlike methods that only imitate observed behavior, learning an explicit reward can support policy re-optimization when the environment dynamics or downstream constraints change~\cite{zeng2022maximum}. We study maximum-likelihood IRL (ML-IRL), which admits the bilevel formulation
\begin{align*}
\min_{\phi}\ \Louter(\theta^\star(\phi))
\qquad
\text{s.t.}\quad
\theta^\star(\phi)\in\arg\min_\theta\Linner(\theta,\phi) \eqsp,
\end{align*}
where the inner problem learns an entropy-regularized policy under reward $r_\phi$, while the outer objective measures its fit to expert demonstrations.

Bilevel IRL can be approached through nested optimization, simultaneous single-loop updates, or differentiation through the inner solution. \cite{zeng2022maximum} propose a single-loop ML-IRL method that alternates one policy-improvement step with a stochastic reward update and enjoys finite-time convergence guarantees. This avoids repeatedly solving the inner problem. A complementary approach is implicit differentiation, which explicitly accounts for the sensitivity of the optimized policy to changes in the reward. Its main bottleneck is an inverse-inner-Hessian-vector product: explicit construction requires quadratic memory, while iterative approximations may require many Hessian-vector products and can be sensitive to conditioning~\cite{ghadimi2018approximation}.

In some problems, the hypergradient can be computed efficiently. For example, Natural Hypergradient Descent (NHGD)~\cite{b1} replaces the inner Hessian with a Fisher-information surrogate for likelihood-based inner problems. However, its standard setting assumes a fixed data distribution, whereas the ML-IRL inner objective is a reverse KL divergence over a policy-induced trajectory distribution, whose sampling distribution depends on the policy parameters. The Hessian--Fisher relation is therefore not immediate. We show that, under exact inner optimality and realizability,
\begin{align*}
\nabla_\theta^2\Linner(\theta^\star(\phi),\phi)
=
\alpha\Fisher_{\theta^\star(\phi)} \eqsp,
\end{align*}
where $\Fisher_{\theta^\star(\phi)}$ is the discounted trajectory Fisher information matrix of the policy. This identity yields a Fisher-based implicit hypergradient for ML-IRL. To avoid the remaining quadratic cost of a dense Fisher matrix, we approximate the required inverse-Fisher-vector product using a streaming SCFD sketch~\cite{chen2020efficient}.

Our contributions are:
\begin{itemize}
\item We derive sample-based estimators for the ML-IRL implicit hypergradient and prove that the inner Hessian is proportional to the trajectory Fisher information matrix at a realizable inner optimum. To our knowledge, this is the first use of this identity for implicit hypergradient computation in ML-IRL.

\item We propose an SCFD-based solver that streams policy-score vectors and avoids constructing either the score matrix or the dense Fisher matrix, reducing storage from \(\mathcal O(d_\theta^2)\) to \(\mathcal O(md_\theta)\) for sketch size \(m\).

\item We compare the method with single-loop ML-IRL on CartPole and LQR under matched computational budgets. The Fisher-based methods achieve competitive policy quality and strong reward-ranking performance, while moderate sketches reduce memory and can improve runtime relative to an explicit Fisher solver.
\end{itemize}

\section{Related Work}

\paragraph{Inverse reinforcement learning.}
Inverse reinforcement learning (IRL) infers a reward function from expert demonstrations. Classical approaches include maximum-margin, Bayesian, maximum-likelihood, and maximum-entropy formulations. Maximum-entropy IRL is especially relevant to our setting because it models trajectory probabilities through cumulative rewards and connects reward learning with entropy-regularized control~\cite{levine2018reinforcement,bloem_maxent}.

The closest related work is the ML-IRL framework of ~\cite{zeng2022maximum}, which formulates reward learning as a bilevel problem with expert-action likelihood at the outer level and entropy-regularized policy optimization at the inner level. For linear rewards, this formulation is dual to maximum-entropy IRL. They propose a single-loop, two-timescale algorithm with finite-time convergence guarantees. We consider a complementary strategy based on implicit differentiation, exploiting inner-level curvature to compute the hypergradient.

Related bilevel RL methods model the dependence of an outer reward-design or alignment objective on a learned policy. PARL applies implicit hypergradients to policy alignment~\cite{chakraborty2024parl}, while ~\cite{shen2025principled} replace lower-level optimality with a penalty formulation, and ~\cite{yang2024bilevel} exploit the soft Bellman fixed point to obtain fully first-order hypergradients. In iterative RLHF, ~\cite{gauthier2026explaining} show that ignoring the policy's influence on subsequent reward-model updates omits an important steering term. These works address broader reward-design and alignment settings, whereas we specialize to ML-IRL and exploit the Fisher geometry of its reverse-KL inner problem.

\paragraph{Hypergradient methods.}
Bilevel hypergradients are commonly estimated either by differentiating through a truncated inner optimization trajectory or by implicit differentiation. The latter requires an inverse-Hessian--vector product, often approximated using Neumann-series estimators or iterative solvers such as conjugate gradients~\cite{ghadimi2018approximation,b1}. Unrolling may require storing the inner optimization trajectory, while implicit approximations generally require repeated Hessian-vector products.

Natural Hypergradient Descent (NHGD)~\cite{b1} exploits likelihood structure by replacing the inner Hessian with a Fisher approximation. Its standard formulation assumes a forward-KL or negative-log-likelihood objective under parameter-independent sampling, whereas ML-IRL involves a reverse KL over policy-induced trajectories. We show that, under exact inner optimality and realizability, the inner Hessian equals $\alpha$ times the discounted trajectory Fisher matrix, adapting Fisher-based hypergradients to ML-IRL. A streaming spectral sketch avoids storing this matrix densely.

\paragraph{Scalable Fisher approximations.}
The Fisher information matrix also underlies natural policy-gradient methods, which precondition policy updates by $\Fisher_\theta^{-1}\nabla_\theta J(\theta)$~\cite{npg_in_rl}. Scalable variants include ACKTR, which uses layerwise Kronecker-factored curvature approximations~\cite{wu2017scalable}, and rank-one approximations to the inverse Fisher matrix~\cite{huo2026rank}. These methods approximate the natural policy gradient rather than a bilevel hypergradient and impose architectural or fixed-rank structure on the Fisher approximation.

Matrix sketching offers a complementary, data-dependent approximation. Frequent Directions~\cite{ghashami2015frequent} maintains a compact spectral approximation of a matrix from streamed rows, while Spectral Compensation Frequent Directions (SCFD)~\cite{chen2020efficient} compensates for spectral information discarded during compression. Related sketches have been used for covariance and inverse-design-matrix approximation in contextual bandits~\cite{b4,lin_ucb}. We apply SCFD to streamed, discounted policy-score vectors and use the resulting sketch to approximate the damped inverse-Fisher-vector product inside the ML-IRL hypergradient, without constructing either the score matrix or the dense Fisher matrix.

\section{Preliminaries}
A \emph{Markov Decision Process} (MDP) is a tuple $(\mcS, \mcA, \mathcal{P}, r)$,
 where $\mcS$ is the state space, $\mcA$ is the action space,
$\mathcal{P}$ is the transition probability kernel and
$r: \mcS \times \mcA \to \rset$ is the reward function.
For each state $s\in \mcS$ and action $a\in \mcA$, $\mcP(\cdot | s, a)$ stands
for a distribution over the states in $\mcS$,
that is, the distribution over the next states given that action $a$ is taken in state $s$.
For each action $a\in \mcA$ and state $s\in \mcS$, $r(s,a)$ gives the reward received when action $a$ is taken in state $s$.
A policy $\pi$ is a map from states to probability distributions
over actions, that is,
\[
    \pi : \mcS \to \Delta(\mcA),
\]
where $\Delta(\mcA)$ denotes the set of probability distributions over
$\mcA$, and $\pi(\cdot\mid s)$ is the distribution over actions in state $s$. An MDP describes the interaction of an agent and its environment. At time $t$, given the current state $X_t\in\mcS$, the agent chooses an action
$A_t\sim\pi(\cdot\mid X_t)$, and the next state is sampled according to $X_{t+1}\sim\mcP(\cdot\mid X_t,A_t)$.

Let $\pi_\theta$ be a differentiable policy parameterized by
$\theta\in\Theta\subseteq\rset^d$. Let $p_{\pi_\theta}$ denote the trajectory distribution induced by $\mcP$ and policy $\pi_\theta$. We define the
\emph{discounted trajectory Fisher information matrix} by
\[
    \Fisher_\theta
    :=
    \E_{\tau \sim p_{\pi_\theta}}
    \left[
        \sum_{t=1}^\infty \gamma^{t-1}\nabla_\theta \log \pi_\theta(a_t \mid s_t)
        \nabla_\theta \log \pi_\theta(a_t\mid s_t)^\top
    \right].
\]

\section{Problem statement}
We consider an infinite-horizon MDP $M = (\mcS, \mcA, \mathcal{P}, \nu, r)$ with initial state distribution \(\nu\) and transition kernel \(\mathcal{P}(s' \mid s,a)\). To incorporate discounting, we use the standard absorbing-state construction.
Introduce an absorbing state \(s_f\) and a fixed action \(a_f\), such that
\[
    \mathcal{P}(s_f\mid s_f,a_f)=1,
    \qquad
    r_\phi(s_f,a_f)=0,
    \qquad
    \pi_\theta(a_f\mid s_f)=1.
\]
For a fixed discount factor \(\gamma\in(0,1)\), define the modified transition kernel
\[
    \widetilde{\mathcal{P}}(s'\mid s,a)
    =
    \gamma \mathcal{P}(s'\mid s,a)
    +
    (1-\gamma)\mathbb{I}\{s'=s_f\}.
\]
We denote the corresponding MDP by $\wt{M} = (\mcS \cup \{s_f\}, \mcA \cup \{a_f\}, \wt{\mathcal{P}}, \nu, r)$. Under this construction, the process continues with probability \(\gamma\) at each
time step and transitions to the absorbing state with probability \(1-\gamma\).
Since both the reward and the log-policy term are zero in the absorbing state, this
construction is equivalent to weighting the contribution of time step \(t\) by
\(\gamma^{t-1}\).

Let \(\pi\) be a policy and let a trajectory be denoted by
\[
    \tau = (s_1,a_1,\ldots,s_{T(\tau)},a_{T(\tau)}, s_f, \dots),
    \qquad
    s_1 \sim \nu \eqsp,
\]
where we have set $T(\tau) = \inf\{t \geq 1: s_{t+1} = s_f\}$. Due to the structure of $\wt{\mathcal{P}}$, it can be easily verified that $\Pr[T(\tau) < \infty] = 1$. The trajectory distribution induced by \(\pi\) in the modified MDP is given by
\[
    \widetilde p_\pi(\tau)
    =
    \tilde{\rho}(\tau)
    \prod_{t=1}^{\infty}
    \pi(a_t \mid s_t)\eqsp, \quad
    \tilde{\rho}(\tau)
    =
    \nu(s_1)
    \prod_{t=1}^{\infty}
    \widetilde{\mathcal{P}}(s_{t+1}\mid s_t,a_t)\eqsp,
\]
where $\tilde{\rho}(\tau)$ denotes the environment-dependent part of the trajectory probability.

We parameterize the state-action reward function by $\phi$. The learned reward assigned to a trajectory is defined as the undiscounted sum of per-step rewards:
\[
R_{\phi}(\tau)
=
\sum_{t=1}^{\infty}
r_\phi(s_t,a_t) \eqsp.
\]
Note that since $T$ is finite almost surely, $R_\phi(\tau)$ is also finite almost surely.


Following the work \cite{levine2018reinforcement}, the learned reward induces a trajectory
distribution that assigns higher probability to trajectories with larger cumulative
reward. Since the numerical scale of the learned reward is not fixed a priori, we
introduce a fixed temperature parameter \(\alpha>0\):
\[
    \tilde{p}_{\phi}(\tau)
    =
    \frac{1}{Z_{\phi}}
    \tilde{\rho}(\tau)
    \exp\left(\frac{R_\phi(\tau)}{\alpha}\right)\eqsp, \quad
    Z_{\phi}
    =
    \int
    \tilde{\rho}(\tau)
    \exp\left(\frac{R_\phi(\tau)}{\alpha}\right)
    d\tau < \infty \eqsp.
\]
Here, \(\alpha\) is treated as a fixed hyperparameter rather than a learnable
reward parameter. Smaller values of \(\alpha\) make the induced distribution more
concentrated around high-reward trajectories, while larger values of \(\alpha\)
produce a smoother distribution. This temperature parameter controls the relative
scale of the reward term in the maximum-entropy trajectory model.

We now assume that the policy is parameterized by $\theta$, i.e. $\pi = \pi_\theta$. Let $\tilde{p}_{\pi_\theta}(\tau)$ denote the trajectory distribution induced by this policy in MDP $\wt{M}$. We also assume access to expert demonstrations generated by an expert policy $\pi_{\mathrm{expert}}$, which represents the behavior of the real object or agent. We denote the corresponding trajectory distribution by $\tilde{p}_{\mathrm{expert}}(\tau)$. The next lemma guarantees that the absorbing-state construction allows expectations under the modified process to be evaluated using trajectories from the original MDP.
\begin{lemma}[Absorbing-state representation of discounting]
\label{lem:absorbing-discounting-main}
Let \(f(s,a)\) be any measurable function such that \(f(s_f,a_f)=0\) and the expectations below are finite. Then
\[
    \E_{\tau\sim \widetilde p_{\pi_\theta}}
    \left[
        \sum_{t=1}^{\infty}
        f(s_t,a_t)
    \right]
    =
    \E_{\tau\sim p_{\pi_\theta}}
    \left[
        \sum_{t=1}^{\infty}
        \gamma^{t-1} f(s_t,a_t)
    \right].
\]
\end{lemma}

Therefore, we formulate the following bilevel optimization problem:
\begin{align}
\label{eq:bilevel_irl_def}
    &\KL
    \left(
        \tilde{p}_{\mathrm{expert}}(\tau)
        \,\|\,
        \tilde{p}_{\pi_{\theta^{*}(\phi)}}(\tau)
    \right)
    \to \min_{\phi} \eqsp,\\
    \nonumber
    &\text{s.t.}\quad
    \theta^*(\phi) \in \arg\min_{\theta}
    \KL
    \left(
        \tilde{p}_{\pi_\theta}(\tau)
        \,\|\,
        \tilde{p}_\phi(\tau)
    \right) \eqsp.
\end{align}
Now, using \Cref{lem:absorbing-discounting-main}, we provide an equivalent formulation of this problem in \Cref{prop:irl_problem_formulation}.


\begin{proposition}
\label{prop:irl_problem_formulation}
Set
\begin{align}
\label{eq:loss_irl_def}
    \Loss_{\mathrm{outer}}(\theta)
&:=
-
\mathbb{E}_{\tau\sim p_{\mathrm{expert}}}
\left[
    \sum_{t=1}^{\infty}
    \gamma^{t-1}
    \log \pi_\theta(a_t\mid s_t)
\right]\eqsp,\\
 \nonumber\Loss_{\mathrm{inner}}(\theta,\phi)
    &:=
    \mathbb{E}_{\tau\sim p_{\pi_\theta}}
    \left[
        \sum_{t=1}^{\infty}
        \gamma^{t-1}
        \left(
            \alpha \log \pi_\theta(a_t\mid s_t)
            -
            r_\phi(s_t,a_t)
        \right)
    \right]\eqsp.
\end{align}
The optimization problem defined in \eqref{eq:bilevel_irl_def} is equivalent to
\begin{equation}
\label{eq:loss_no_kl_def}
    \Loss_{\mathrm{outer}}(\theta^*(\phi)) \to \min_{\phi}
    \quad \text{s.t.} \quad
    \theta^*(\phi) \in \arg\min_{\theta}
    \Loss_{\mathrm{inner}}(\theta,\phi) \eqsp.
\end{equation}
\end{proposition}
The proof of \Cref{prop:irl_problem_formulation} is provided in \Cref{app:discounted-formulation}. This type of bilevel IRL problem formulation is not novel and has appeared in, e.g., \cite{zeng2022maximum}. Here, the inner objective is equivalent to entropy-regularized RL, and the outer objective evaluates a policy on expert trajectories. The KL formulation of this problem is primarily interesting in the context of hypergradient descent because it helps us derive explicit formulas for the hypergradient, as described further in this section.

We consider a gradient-based approach to solve \eqref{eq:loss_no_kl_def}. Computing the gradient $\nabla_\phi \Loss_{\mathrm{outer}}(\theta^*(\phi))$ is not straightforward, since \(\Loss_{\mathrm{outer}}\) depends on \(\phi\) implicitly through the inner solution \(\theta^*(\phi)\). Changing \(\phi\) changes the reward-induced distribution \(\tilde p_\phi\), which changes the solution of the inner optimization problem and, in turn, the value of the outer objective. However, it is possible to derive
\begin{equation}
\label{eq:implicit-hypergradient}
    \nabla_\phi \Loss_{\mathrm{outer}}(\theta^\star(\phi))
    =
    -
    \left[\nabla_{\theta,\phi}\Linner(\theta^\star(\phi),\phi)\right]^\top
    \left[
        \nabla_{\theta}^2\Linner(\theta^\star(\phi),\phi)
    \right]^{-1}
    \nabla_{\theta}\Louter(\theta^\star(\phi))
    \eqsp,
\end{equation}
which, for completeness, we prove in \Cref{app:hypergradient-derivations}. In general, this gradient is impractical to compute since it requires estimating the inverse Hessian of the inner loss. However, in some problems it can be computed efficiently by exploiting the structure of the loss, for instance, when the Hessian can be represented by a Fisher information matrix \cite{b1}. In our setting, however, the inner problem involves reverse-KL minimization. Nevertheless, it can be shown that at the optimum $\theta^*(\phi)$, the Hessian of the inner loss is proportional to the Fisher information matrix. We establish this result in \Cref{prop:inner-hessian-fisher-main}.
\begin{proposition}
\label{prop:inner-hessian-fisher-main}
Assume that the inner problem is solved exactly and that the policy class is rich enough so that
\[
    \KL
    \left(
        \widetilde p_{\pi_{\theta^*(\phi)}}(\tau)
        \,\|\,
        \widetilde p_\phi(\tau)
    \right)
    =
    0 \eqsp.
\]
Then the Hessian of the inner objective at the inner optimum is proportional to the Fisher information matrix of the policy-induced trajectory distribution:
\begin{align*}
    \nabla_\theta^2 \Linner(\theta^\star(\phi),\phi) = \alpha \Fisher_{\theta^\star(\phi)} \eqsp.
\end{align*}
\end{proposition}
The proof of \Cref{prop:inner-hessian-fisher-main} is provided in \Cref{app:discounted-formulation}. Thus, when the inner solution is sufficiently close to the optimum, this result motivates approximating the Hessian using the Fisher information matrix. Given a batch of trajectories $\mathcal{D}_{\text{agent}}$ from $p_{\pi_\theta}$ of size $N_{\text{agent}}$, we can estimate the Fisher information matrix as
\begin{align}
\label{eq:method-empirical-fisher}
    \widehat{\mF}_\theta
    =
    \frac{1}{N_{\text{agent}}}
    \sum_{\tau\in\mathcal D_{\mathrm{agent}}}
    \sum_{t=1}^{|\tau|}
    \gamma^{t-1}
    \nabla_\theta\log\pi_\theta(a_t\mid s_t)
    \nabla_\theta\log\pi_\theta(a_t\mid s_t)^\top .
\end{align}
This representation allows us to estimate the inverse matrix efficiently using rank-one updates, as described in the next section.

\section{Proposed Method}

In this section, we describe how the implicit hypergradient in
\Eqref{eq:implicit-hypergradient} can be evaluated efficiently from sampled
trajectories. Its computation requires three quantities: the outer gradient,
the inverse-inner-Hessian-vector product, and the mixed-derivative-vector
product. We derive practical estimators for each of them and then introduce a
sketching-based approximation that avoids explicit construction of the Fisher
matrix.

\subsection{Fisher-Based Hypergradient Descent}
\label{subsec:explicit-fisher}

We first consider the outer-gradient term in
\Eqref{eq:implicit-hypergradient}. Differentiating the discounted outer
objective in \Eqref{eq:loss_irl_def} with respect to the policy parameters gives
\[
    \nabla_\theta\Louter(\theta)
    =
    -
    \E_{\tau\sim p_{\mathrm{expert}}}
    \left[
        \sum_{t=1}^{\infty}
        \gamma^{t-1}
        \nabla_\theta
        \log\pi_\theta(a_t\mid s_t)
    \right].
\]
Therefore, given expert demonstrations $\mathcal{D}_{\text{expert}}$ from $p_{\text{expert}}$ of size $N_{\text{expert}}$, we use the Monte Carlo estimate
\begin{equation}
\label{eq:method-outer-gradient}
    \widehat{\vg}_\theta
    =
    -
    \frac{1}{N_{\text{expert}}}
    \sum_{\tau\in\mathcal D_{\mathrm{expert}}}
    \sum_{t=1}^{|\tau|}
    \gamma^{t-1}
    \nabla_\theta
    \log\pi_\theta(a_t\mid s_t).
\end{equation}
The corresponding discounted identity is proved in
\Cref{prop:discounted-outer-gradient}.

The second term in \Eqref{eq:implicit-hypergradient} involves the inverse
Hessian of the inner objective. By
\Cref{prop:inner-hessian-fisher-main}, at the exact inner optimum this Hessian
is proportional to the Fisher information matrix. Furthermore, its discounted
trajectory form admits the stepwise representation
\[
    \nabla_\theta^2\Linner(\theta^\star(\phi),\phi)
    =
    \alpha
    \E_{\tau\sim p_{\pi_{\theta^\star(\phi)}}}
    \left[
        \sum_{t=1}^{\infty}
        \gamma^{t-1}
        \nabla_\theta\log\pi_\theta(a_t\mid s_t)
        \nabla_\theta\log\pi_\theta(a_t\mid s_t)^\top
    \right].
\]
Its derivation is provided in
\Cref{prop:discounted-stepwise-trajectory-fisher}. 

The hypergradient does not require the inverse Fisher matrix itself, but only
its product with the outer gradient. We therefore introduce
$\widehat{\vv}$ as the solution of the damped Fisher system
\begin{equation}
\label{eq:method-fisher-system}
    \left(
        \alpha\widehat{\mF}_\theta+\lambda\mI
    \right)
    \widehat{\vv}_\theta
    =
    \widehat{\vg}_\theta,
\end{equation}
where $\lambda>0$ is a regularization parameter.

It remains to evaluate the mixed-derivative term in
\Eqref{eq:implicit-hypergradient}. Direct construction of the mixed derivative
matrix is unnecessary. For each trajectory and time step, define the prefix
score-vector product
\[
    b_{\theta,t}(\tau)
    =
    \sum_{k=1}^{t}
    \nabla_\theta
    \log\pi_\theta(a_k\mid s_k)^\top
    \widehat{\vv}_\theta.
\]
The mixed-derivative-vector product is then estimated as
\begin{equation}
\label{eq:method-cross-vector}
    \widehat{\vu}
    =
    -
    \frac{1}{N_{\text{agent}}}
    \sum_{\tau\in\mathcal D_{\mathrm{agent}}}
    \nabla_\phi
    \left[
        \sum_{t=1}^{|\tau|}
        \gamma^{t-1}
        r_\phi(s_t,a_t)b_t(\tau)
    \right].
\end{equation}
The derivation is given in
\Cref{prop:discounted-prefix-score-inner-cross-derivative} and
\Cref{cor:jvp-inner-cross-derivative}. Computationally, all products
$\nabla_\theta\log\pi_\theta(a_t\mid s_t)^\top\widehat{\vv}$ are obtained
without explicitly forming the policy-score Jacobian, and the prefix values
$b_t(\tau)$ are then computed by a single cumulative-sum pass.

Combining these components gives $\widehat{\vh}=-\widehat{\vu}$ and the update
$\phi\gets\phi-\eta_\phi\widehat{\vh}$. The complete procedure is summarized
in \Cref{alg:explicitfisherhd}.

\HypergradientDescentWithExplicitFisher

\subsection{Hypergradient Descent with Fisher Sketching}
\label{subsec:fisher-with-sketching}

The explicit formulation requires storing
$\widehat{\mF}\in\R^{d_\theta\times d_\theta}$, resulting in
$\mathcal O(d_\theta^2)$ memory complexity. To avoid constructing $\widehat{\mF}_\theta$ explicitly, we instead represent the scaled Fisher matrix as $\alpha\widehat{\mF}_\theta=\mX^\top\mX$, where the rows of $\mX$ are constructed
from the weighted policy score vectors. For every transition from an agent trajectory, define
\begin{equation}
\label{eq:method-fisher-row}
    \vx_t(\tau)
    =
    \sqrt{
        \frac{\alpha}{N_{\text{agent}}}
        \gamma^{t-1}
    }
    \nabla_\theta
    \log\pi_\theta(a_t\mid s_t).
\end{equation}
Stacking these vectors as rows of a matrix $\mX$ gives
$\alpha\widehat{\mF}_\theta=\mX^\top\mX$, so \Eqref{eq:method-fisher-system} becomes
\begin{equation}
\label{eq:method-fisher-sketch-system}
    \left(
        \mX^\top\mX+\lambda\mI
    \right)
    \widehat{\vv}
    =
    \widehat{\vg}.
\end{equation}
This representation allows the rows of $\mX$ to be generated sequentially and immediately passed to a matrix sketch, without materializing either $\mX$ or $\widehat{\mF}$. We employ Spectral Compensation Frequent Directions (SCFD)~\cite{chen2020efficient} to maintain a compact approximation and approximately solve \Eqref{eq:method-fisher-sketch-system}. In our setting, SCFD is applied to the sequentially generated weighted policy-score vectors defined in \Eqref{eq:method-fisher-row}. For completeness, the SCFD update and solve procedures are provided in \Cref{alg:scfd} in the appendix. After all agent trajectories have been processed, the solver is queried once to obtain $\widehat{\vv}$, while the remaining hypergradient computation is unchanged. The resulting sketched hypergradient descent method is summarized in \Cref{alg:sketched-fishernhd}.

\HypergradientDescentWithFisherSketching

\section{Numerical Results}
\label{sec:numerical-results}

We evaluate the proposed approach on two control environments:
CartPole~\cite{barto1983neuronlike} and LQR, a continuous-control environment based on the classical linear-quadratic regulator problem~\cite{anderson2007optimal}. CartPole is a low-dimensional benchmark with discrete actions; LQR is a continuous-control setting with a larger policy parameterization. Unless stated otherwise, all methods within each environment share the same policy and reward architectures, expert data, optimization settings, and computational budget.

We compare against ML-IRL~\cite{zeng2022maximum} as the bilevel IRL baseline and consider two variants of our Fisher-based hypergradient: \emph{Explicit Fisher}, which explicitly constructs the empirical Fisher matrix and solves the damped linear system of \Secref{subsec:explicit-fisher}, and \emph{Fisher with Sketching}, which replaces this with the sketched approximation of \Secref{subsec:fisher-with-sketching}. We study both optimization behavior and computational efficiency (wall-clock time and peak memory).

\subsection{Stability and Efficiency of the Fisher Solver}

\textbf{Effect of Fisher Damping.}
We first study the effect of the damping parameter $\lambda$ in
\Eqref{eq:method-fisher-system}, using the explicit Fisher formulation to
isolate the effect of regularization from sketching error. We perform a
grid search over $\lambda$ and monitor the outer objective.

\begin{figure}[!htbp]
    \centering
    \begin{minipage}[c]{0.64\linewidth}
        \centering
        \includegraphics[width=0.9\linewidth]
        {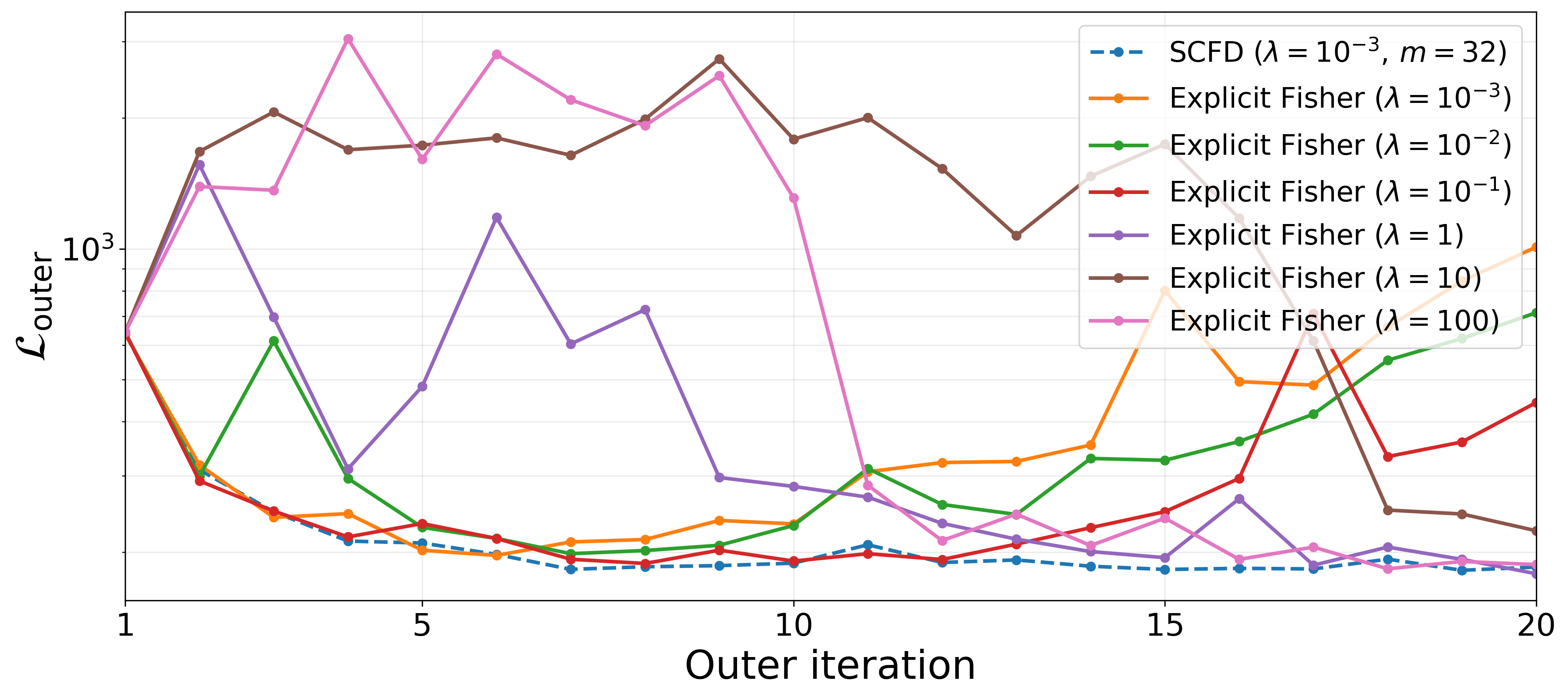}
    \end{minipage}
    \hfill
    \begin{minipage}[c]{0.35\linewidth}
        \centering
        \scriptsize
        \begin{tabular}{c c}
            \toprule
            $\lambda$ & $\Louter$ \\
            \midrule
            $10^{-3}$ & $532.75 \pm 145.02$ \\
            $10^{-2}$ & $205.84 \pm 34.91$ \\
            $10^{-1}$ & $427.76 \pm 166.27$ \\
            $10^{0}$  & $\mathbf{185.27 \pm 4.38}$ \\
            $10^{1}$  & $501.83 \pm 410.30$ \\
            $10^{2}$  & $698.50 \pm 227.01$ \\
            \bottomrule
        \end{tabular}
    \end{minipage}
    \caption{
    Effect of the Fisher damping parameter $\lambda$ on optimization.
    Left: validation outer loss throughout training.
    Right: $\Louter$ averaged over the final five outer iterations.
    }
    \label{fig:fisher-reg-grid}
\end{figure}

The results reveal a stability–fidelity trade-off. Weak damping leaves the
Fisher system poorly conditioned, destabilizing the outer optimization;
strong damping stabilizes the solve but drives
$(\alpha\widehat{\mF}_\theta+\lambda\mI)^{-1} \approx (\lambda\mI)^{-1}$,
so the Fisher information contributes little to the update while remaining
expensive to compute. The explicit solver must therefore trade off
numerical stability against preserving the Fisher geometry.

\textbf{Effect of the SCFD Sketch Size.}
We fix $\lambda=10^{-3}$, since the sketched solver remains stable under
substantially weaker damping than the explicit formulation, and vary the
sketch size $m$.

\begin{figure}[!htbp]
    \centering
    \begin{minipage}[c]{0.64\linewidth}
        \centering
        \includegraphics[width=0.9\linewidth]{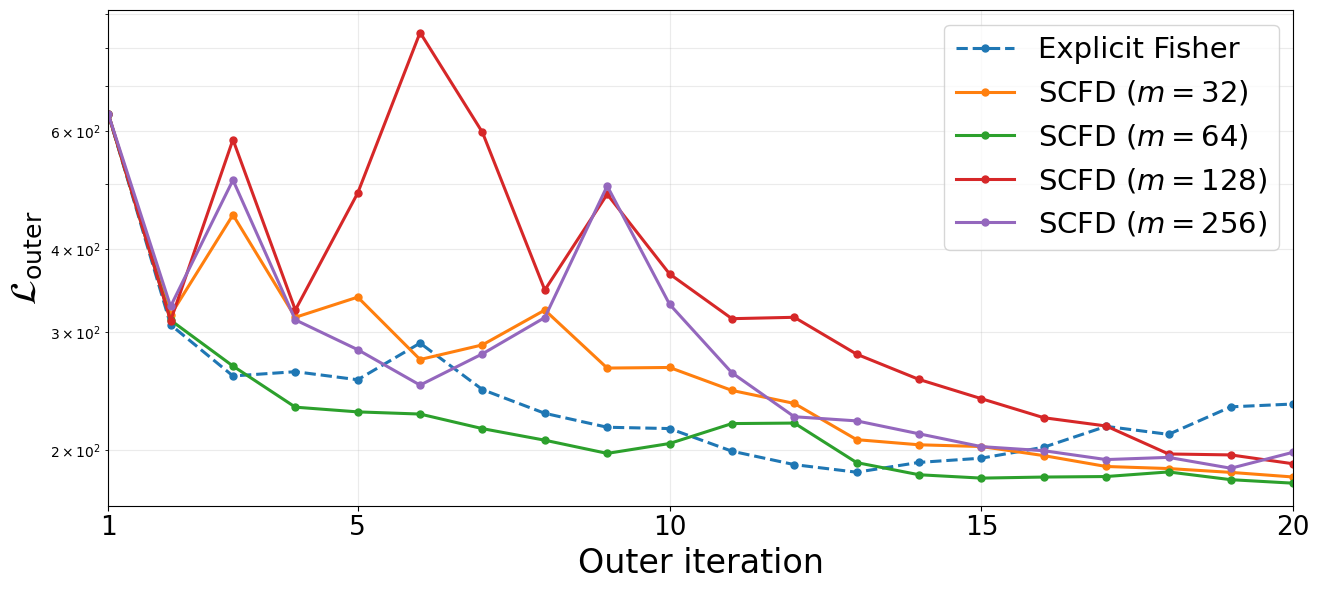}
    \end{minipage}
    \hfill
    \begin{minipage}[c]{0.35\linewidth}
        \centering
        \scriptsize
        \setlength{\tabcolsep}{3pt}
        \begin{tabular}{c c c}
            \toprule
            Method & $\Louter$ & Speedup \\
            \midrule
            Explicit & $219.15 \pm 13.82$ & $1.00\times$ \\
            $m=32$   & $187.90 \pm 5.22$  & $1.15\times$ \\
            $m=64$   & $\mathbf{181.70 \pm 2.60}$ & $\mathbf{1.17\times}$ \\
            $m=128$  & $204.96 \pm 14.38$ & $1.13\times$ \\
            $m=256$  & $194.73 \pm 4.64$  & $0.78\times$ \\
            \bottomrule
        \end{tabular}
    \end{minipage}
    \caption{
        Effect of the SCFD sketch size on optimization for $\lambda=10^{-3}$.
        Left: validation outer loss for Explicit Fisher and SCFD at
        different sketch sizes. Right: $\Louter$ averaged over the final
        five outer iterations, with speedup relative to Explicit Fisher.
    }
    \label{fig:fisher-sketch-grid}
\end{figure}

A small sketch already suffices to match or exceed the optimization quality
of Explicit Fisher: all sketch sizes achieve a lower final outer loss, with
$m=64$ performing best. Larger sketches do not improve quality
monotonically and increase cost; $m=256$ is even slower than Explicit
Fisher despite a lower loss. Moderate sketch sizes therefore offer the best
trade-off, with $m=64$ attaining the lowest outer objective while remaining
faster than the explicit solver.

\subsection{Computational Efficiency}
\label{subsec:computational-efficiency}

As shown in Table~\ref{tab:computational-efficiency}, sketching provides
the largest benefit in the higher-dimensional LQR setting, reducing
peak memory by up to $1.31\times$; small sketches also preserve or improve
runtime, while larger sketches erode this advantage. On CartPole, memory
usage is essentially unchanged, but sketching yields a $1.29\times$
speedup. The benefit of sketching thus grows with the cost of the Fisher
representation.

\begin{table}[!htbp]
    \centering
    \caption{
        Computational efficiency of Fisher with Sketching relative to
        Explicit Fisher. Memory gain is the ratio of peak memory usage,
        time speedup the ratio of wall-clock time per outer iteration.
        Values above $1\times$ indicate an improvement over Explicit
        Fisher.
    }
    \label{tab:computational-efficiency}
    \scriptsize
    \setlength{\tabcolsep}{10.0pt}
    \renewcommand{\arraystretch}{1.0}
    \begin{tabular}{l c c c}
        \toprule
        Environment & Method & Memory gain & Time speedup \\
        \midrule
        CartPole
            & Explicit Fisher & $1.00\times$ & $1.00\times$ \\
        & Sketching ($m=8$) & $0.97\times$ & $\mathbf{1.29\times}$ \\
        \midrule
        LQR
            & Explicit Fisher & $1.00\times$ & $1.00\times$ \\
        & Sketching ($m=32$)  & $1.28\times$ & $\mathbf{1.04\times}$ \\
        & Sketching ($m=64$)  & $\mathbf{1.31\times}$ & $0.89\times$ \\
        & Sketching ($m=128$) & $1.21\times$ & $0.99\times$ \\
        & Sketching ($m=256$) & $1.10\times$ & $0.79\times$ \\
        \bottomrule
    \end{tabular}
\end{table}

\vspace{-0.3cm}

\subsection{Reward Comparison}
\label{sec:overall-irl-performance}

We evaluate the recovered rewards under a matched 24-hour training budget,
using identical architectures within each environment and separately tuned
hyperparameters; a fresh policy is then trained from scratch on each
recovered reward under a fixed RL budget. We report \textbf{PolicyNLL}
($\downarrow$, negative log-likelihood of expert actions under the induced
policy), \textbf{EnvReturn} ($\uparrow$, return of this policy under the
ground-truth reward), and \textbf{RankCorr} ($\uparrow$, rank correlation
between learned and ground-truth trajectory returns). \textbf{ExpertRet}
and \textbf{RandomRet} give the ground-truth returns of the expert and
random policies as reference.

\begin{table}[!htbp]
    \centering
    \caption{
        Final IRL performance under a matched computational budget. Arrows indicate the preferred direction.
    }
    \label{tab:final-irl-performance}
    \scriptsize
    \setlength{\tabcolsep}{5.0pt}
    \renewcommand{\arraystretch}{1.0}
    \begin{tabular}{l c c c c c}
        \toprule
        Method
        & PolicyNLL $\downarrow$
        & EnvReturn $\uparrow$
        & RankCorr $\uparrow$
        & ExpertRet
        & RandomRet \\
        \midrule
        \multicolumn{6}{c}{\textbf{CartPole}} \\
        \midrule
        ML-IRL
        & $331.20 \pm 0.23$
        & $\mathbf{500.00 \pm 0.00}$
        & $0.892 \pm 0.009$
        & \multirow{3}{*}{$500.00 \pm 0.00$}
        & \multirow{3}{*}{$22.47 \pm 0.99$} \\
        Explicit Fisher
        & $281.53 \pm 0.21$
        & $\mathbf{500.00 \pm 0.00}$
        & $\mathbf{0.901 \pm 0.007}$
        & & \\
        Fisher with Sketching
        & $\mathbf{183.35 \pm 1.66}$
        & $\mathbf{500.00 \pm 0.00}$
        & $0.870 \pm 0.011$
        & & \\
        \midrule\midrule
        \multicolumn{6}{c}{\textbf{LQR}} \\
        \midrule
        ML-IRL
        & $\mathbf{248.25 \pm 2.33}$
        & $\mathbf{-25.38 \pm 1.32}$
        & $0.908 \pm 0.012$
        & \multirow{3}{*}{$-23.23 \pm 1.16$}
        & \multirow{3}{*}{$-10828.62 \pm 721.99$} \\
        Explicit Fisher
        & $284.57 \pm 3.24$
        & $-36.82 \pm 1.82$
        & $0.953 \pm 0.007$
        & & \\
        Fisher with Sketching
        & $252.37 \pm 1.33$
        & $-28.58 \pm 1.29$
        & $\mathbf{0.972 \pm 0.005}$
        & & \\
        \bottomrule
    \end{tabular}
\end{table}

On LQR, ML-IRL achieves the best PolicyNLL and EnvReturn, though its
margin over Fisher with Sketching is small, while the latter attains the
highest RankCorr. On CartPole, all methods reach expert-level EnvReturn;
Fisher with Sketching substantially improves PolicyNLL, while Explicit
Fisher achieves the best RankCorr. Overall, the Fisher-based methods remain
competitive in policy quality while providing strong reward-ranking
performance.

Figure~\ref{fig:reward} shows that Fisher with Sketching reaches competitive validation performance early in training. Despite using far fewer outer updates, it closely matches ML-IRL in PolicyNLL, EnvReturn, and RankCorr within the same wall-clock window, whereas Explicit Fisher converges substantially more slowly.

\begin{figure}[!htbp]
    \centering
    \includegraphics[width=1.00\linewidth]{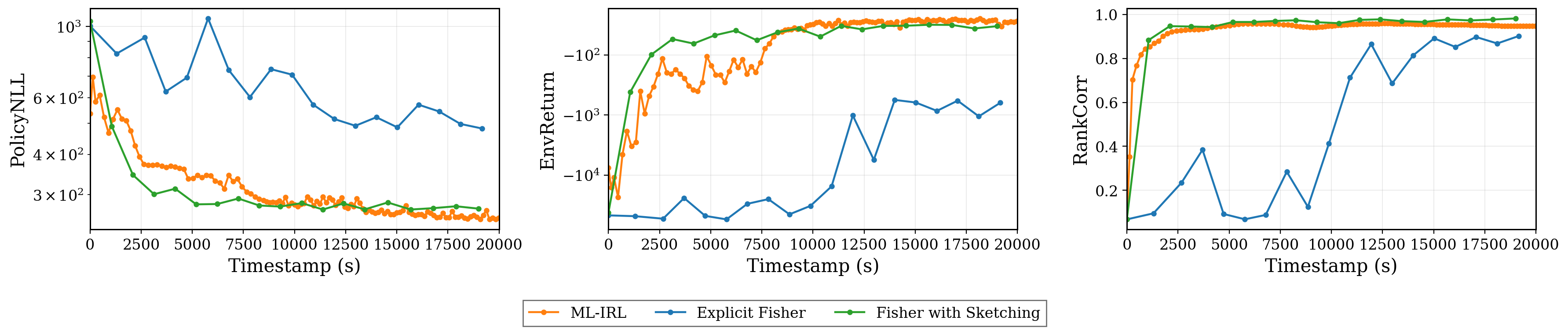}
    \caption{PolicyNLL, EnvReturn, and RankCorr over wall-clock time on the LQR validation set, truncated to the first 20{,}000 seconds. Fisher with Sketching matches ML-IRL's performance using an order of magnitude fewer iterations.}
    \label{fig:reward}
\end{figure}

\vspace{-0.3cm}

\section{Conclusion}
\label{sec:conclusion}

We developed a Fisher-based implicit hypergradient method for bilevel IRL and established that, under exact inner optimality and realizability, the inner Hessian equals $\alpha$ times the discounted trajectory Fisher matrix. An SCFD-based streaming approximation reduces curvature storage from $\mathcal O(d_\theta^2)$ to $\mathcal O(md_\theta)$. Experiments on two control tasks indicate that moderate sketches improve the stability--efficiency trade-off relative to the dense Fisher solver while remaining competitive with ML-IRL under matched computational budgets.

\bibliography{icomp2026_conference}

@article{b1,
  title         = {Natural Hypergradient Descent: Algorithm Design, Convergence Analysis, and Parallel Implementation},
  author        = {Kong, Deyi and Chen, Zaiwei and Zhang, Shuzhong and Mou, Shancong},
  journal       = {arXiv preprint arXiv:2602.10905},
  year          = {2026},
  eprint        = {2602.10905},
  archivePrefix = {arXiv}
}

@article{wu2017scalable,
  title={Scalable trust-region method for deep reinforcement learning using kronecker-factored approximation},
  author={Wu, Yuhuai and Mansimov, Elman and Grosse, Roger B and Liao, Shun and Ba, Jimmy},
  journal={Advances in neural information processing systems},
  volume={30},
  year={2017}
}

@article{yang2024bilevel,
  title={Bilevel reinforcement learning via the development of hyper-gradient without lower-level convexity},
  author={Yang, Yan and Gao, Bin and Yuan, Ya-xiang},
  journal={arXiv preprint arXiv:2405.19697},
  year={2024}
}

@article{shen2025principled,
  title={Principled penalty-based methods for bilevel reinforcement learning and rlhf},
  author={Shen, Han and Yang, Zhuoran and Chen, Tianyi},
  journal={Journal of Machine Learning Research},
  volume={26},
  number={114},
  pages={1--49},
  year={2025}
}

@INPROCEEDINGS{bloem_maxent,
  author={Bloem, Michael and Bambos, Nicholas},
  booktitle={53rd IEEE Conference on Decision and Control}, 
  title={Infinite time horizon maximum causal entropy inverse reinforcement learning}, 
  year={2014},
  volume={},
  number={},
  pages={4911-4916},
  doi={10.1109/CDC.2014.7040156}}

@article{huo2026rank,
  title={Rank-1 Approximation of Inverse Fisher for Natural Policy Gradients in Deep Reinforcement Learning},
  author={Huo, Yingxiao and Dash, Satya Prakash and Stoican, Radu and Kaski, Samuel and Sun, Mingfei},
  journal={arXiv preprint arXiv:2601.18626},
  year={2026}
}

@article{ghashami2015frequent,
  title={Frequent directions: Simple and deterministic matrix sketching},
  author={Ghashami, Mina and Liberty, Edo and Phillips, Jeff M and Woodruff, David P},
  journal={arXiv preprint arXiv:1501.01711},
  year={2015}
}

@article{ghadimi2018approximation,
  title={Approximation methods for bilevel programming},
  author={Ghadimi, Saeed and Wang, Mengdi},
  journal={arXiv preprint arXiv:1802.02246},
  year={2018}
}

@article{gauthier2026explaining,
  title={Explaining and Preventing Alignment Collapse in Iterative RLHF},
  author={Gauthier, Etienne and Bach, Francis and Jordan, Michael I},
  journal={arXiv preprint arXiv:2605.04266},
  year={2026}
}

@inproceedings{chakraborty2024parl,
  title={Parl: A unified framework for policy alignment in reinforcement learning from human feedback},
  author={Chakraborty, Souradip and Bedi, Amrit and Koppel, Alec and Wang, Huazheng and Manocha, Dinesh and Wang, Mengdi and Huang, Furong},
  booktitle={International Conference on Learning Representations},
  volume={2024},
  pages={24410--24449},
  year={2024}
}

@article{zeng2022maximum,
  title={Maximum-likelihood inverse reinforcement learning with finite-time guarantees},
  author={Zeng, Siliang and Li, Chenliang and Garcia, Alfredo and Hong, Mingyi},
  journal={Advances in Neural Information Processing Systems},
  volume={35},
  pages={10122--10135},
  year={2022}
}

@article{npg_in_rl,
  title         = {Natural Policy Gradients in Reinforcement Learning Explained},
  author        = {van Heeswijk, W. J. A.},
  journal       = {arXiv preprint arXiv:2209.01820},
  year          = {2022},
  eprint        = {2209.01820},
  archivePrefix = {arXiv}
}

@inproceedings{b4,
  title     = {Efficient Linear Bandits through Matrix Sketching},
  author    = {Kuzborskij, Ilja and Cella, Leonardo and Cesa-Bianchi, Nicol{\`o}},
  booktitle = {Proceedings of the 22nd International Conference on Artificial Intelligence and Statistics},
  series    = {Proceedings of Machine Learning Research},
  volume    = {89},
  pages     = {177--185},
  address   = {Naha, Okinawa, Japan},
  year      = {2019}
}

@inproceedings{chen2020efficient,
  title     = {Efficient and Robust High-Dimensional Linear Contextual Bandits},
  author    = {Chen, Cheng and Luo, Luo and Zhang, Weinan and Yu, Yong and Lian, Yijiang},
  booktitle = {Proceedings of the Twenty-Ninth International Joint Conference on Artificial Intelligence},
  series    = {IJCAI-20},
  pages     = {4259--4265},
  address   = {Yokohama, Japan},
  year      = {2020}
}

@article{lin_ucb,
  title         = {Scalable LinUCB: Low-Rank Design Matrix Updates for Recommenders with Large Action Spaces},
  author        = {Shustova, Ekaterina and Sheshukova, Marina and Samsonov, Sergey and Frolov, Evgeny},
  journal       = {arXiv preprint arXiv:2510.19349},
  year          = {2025},
  eprint        = {2510.19349},
  archivePrefix = {arXiv}
}

@article{levine2018reinforcement,
  author  = {Levine, Sergey},
  title   = {Reinforcement Learning and Control as Probabilistic Inference: Tutorial and Review},
  journal = {arXiv preprint arXiv:1805.00909},
  year    = {2018}
}

@article{barto1983neuronlike,
  title={Neuronlike Adaptive Elements That Can Solve Difficult Learning Control Problems},
  author={Barto, Andrew G. and Sutton, Richard S. and Anderson, Charles W.},
  journal={IEEE Transactions on Systems, Man, and Cybernetics},
  volume={13},
  number={5},
  pages={834--846},
  year={1983}
}

@book{anderson2007optimal,
  title={Optimal Control: Linear Quadratic Methods},
  author={Anderson, Brian D. O. and Moore, John B.},
  year={2007},
  publisher={Dover Publications}
}
\bibliographystyle{icomp2026_conference}

\appendix

\section{Hypergradient derivations}
\label{app:hypergradient-derivations}

\begin{proposition}[Implicit hypergradient for bilevel IRL]
\label{app:hypergradient-derivation}
Assume that \(\theta^\star(\phi)\) is a differentiable local solution of the inner problem,
\[
    \theta^\star(\phi)
    \in
    \arg\min_{\theta}
    \Linner(\theta,\phi),
\]
and that the Hessian
\[
    \left.
    \frac{\partial^2 \Linner}{\partial \theta^2}
    \right|_{\theta^\star(\phi),\phi}
\]
is invertible. Then the gradient of the induced outer objective
\[
    \Loutertilde(\phi)
    \coloneqq
    \Louter(\theta^\star(\phi))
\]
is given by
\[
\left.
\nabla_{\phi} \Loutertilde 
\right|_{\phi}
=
-
\left.
\frac{\partial^2 \Linner}{\partial \phi \partial \theta}
\right|_{\theta^\star(\phi),\phi}
\left[
\left.
\frac{\partial^2 \Linner}{\partial \theta^2}
\right|_{\theta^\star(\phi),\phi}
\right]^{-1}
\left.
\nabla_{\theta} \Louter
\right|_{\theta^\star(\phi)} \eqsp.
\]
\end{proposition}

\begin{proof}

By definition, the outer objective optimized with respect to the reward parameters is
\[
    \Loutertilde(\phi)
    \coloneqq
    \Louter(\theta^\star(\phi)), \qquad
    \theta^\star(\phi)
    \in
    \arg\min_{\theta}
    \Linner(\theta,\phi).
\]
By the chain rule:
\[
    \left.
    \frac{d\Loutertilde}{d\phi}
    \right|_{\phi}
    =
    \left.
    \frac{\partial \Louter}{\partial \theta}
    \right|_{\theta^\star(\phi)}
    \left.
    \frac{d\theta^\star}{d\phi}
    \right|_{\phi}.
\]
Define the inner first-order optimality condition as:
\[
    g(\theta,\phi)
    \coloneqq
    \frac{\partial \Linner}{\partial \theta}(\theta,\phi).
\]
Since \(\theta^\star(\phi)\) is a local solution of the inner problem, it satisfies:
\[
    g(\theta^\star(\phi),\phi)=0.
\]
Differentiating this identity with respect to \(\phi\), we obtain:
\[
    \frac{dg}{d\phi} (\theta^\star(\phi),\phi) = 0.
\]
Applying the chain rule gives
\[
    \left.
    \frac{\partial g}{\partial \theta}
    \right|_{\theta^\star(\phi),\phi}
    \left.
    \frac{d\theta^\star}{d\phi}
    \right|_{\phi}
    +
    \left.
    \frac{\partial g}{\partial \phi}
    \right|_{\theta^\star(\phi),\phi}
    =
    0.
\]
Therefore,
\[
    \left.
    \frac{d\theta^\star}{d\phi}
    \right|_{\phi}
    =
    -
    \left[
    \left.
    \frac{\partial g}{\partial \theta}
    \right|_{\theta^\star(\phi),\phi}
    \right]^{-1}
    \left.
    \frac{\partial g}{\partial \phi}
    \right|_{\theta^\star(\phi),\phi}.
\]
Since
\[
    \frac{\partial g}{\partial \theta}
    =
    \frac{\partial^2 \Linner}{\partial \theta^2},
    \qquad
    \frac{\partial g}{\partial \phi}
    =
    \frac{\partial^2 \Linner}{\partial \theta \partial \phi},
\]
we get
\[
    \left.
    \frac{d\theta^\star}{d\phi}
    \right|_{\phi}
    =
    -
    \left[
    \left.
    \frac{\partial^2 \Linner}{\partial \theta^2}
    \right|_{\theta^\star(\phi),\phi}
    \right]^{-1}
    \left.
    \frac{\partial^2 \Linner}{\partial \theta \partial \phi}
    \right|_{\theta^\star(\phi),\phi}.
\]
Substituting this expression into the derivative of the outer objective gives
\[
    \left.
    \frac{d\Loutertilde}{d\phi}
    \right|_{\phi}
    =
    -
    \left.
    \frac{\partial \Louter}{\partial \theta}
    \right|_{\theta^\star(\phi)}
    \left[
    \left.
    \frac{\partial^2 \Linner}{\partial \theta^2}
    \right|_{\theta^\star(\phi),\phi}
    \right]^{-1}
    \left.
    \frac{\partial^2 \Linner}{\partial \theta \partial \phi}
    \right|_{\theta^\star(\phi),\phi}.
\]
Finally, switching from derivatives to gradients gives
\[
    \left.
    \nabla_\phi \Loutertilde
    \right|_{\phi}
    =
    \left(
    \left.
    \frac{d\Loutertilde}{d\phi}
    \right|_{\phi}
    \right)^\top .
\]
Moreover, transposing the mixed derivative changes the order of differentiation
\[
    \left(
    \left.
    \frac{\partial^2 \Linner}{\partial \theta \partial \phi}
    \right|_{\theta^\star(\phi),\phi}
    \right)^\top
    =
    \left.
    \frac{\partial^2 \Linner}{\partial \phi \partial \theta}
    \right|_{\theta^\star(\phi),\phi}.
\]
Using also the symmetry of the Hessian,
\[
    \left(
    \left.
    \frac{\partial^2 \Linner}{\partial \theta^2}
    \right|_{\theta^\star(\phi),\phi}
    \right)^\top
    =
    \left.
    \frac{\partial^2 \Linner}{\partial \theta^2}
    \right|_{\theta^\star(\phi),\phi},
\]
we obtain
\[
\left.
\nabla_{\phi} \Loutertilde
\right|_{\phi}
=
-
\left.
\frac{\partial^2 \Linner}{\partial \phi \partial \theta}
\right|_{\theta^\star(\phi),\phi}
\left[
\left.
\frac{\partial^2 \Linner}{\partial \theta^2}
\right|_{\theta^\star(\phi),\phi}
\right]^{-1}
\left.
\nabla_{\theta} \Louter
\right|_{\theta^\star(\phi)}.
\]
\end{proof}

\section{Discounted infinite-horizon derivations}
\label{app:discounted-formulation}

Throughout this appendix, $p_{\pi_\theta}$ denotes the infinite-horizon
trajectory distribution induced by the original transition kernel $\mathcal P$,
whereas $\widetilde p_{\pi_\theta}$ denotes the distribution induced by the
absorbing-state kernel $\widetilde{\mathcal P}$ defined in the main text. We
assume throughout that the relevant random variables are integrable and that
differentiation, expectation, and the discounted infinite sums can be
interchanged. In particular, the series below are assumed to be absolutely
convergent.

\begin{proof}[Proof of \Cref{lem:absorbing-discounting-main}]
Under $\widetilde{\mathcal P}$, the process follows the original dynamics at
each step with probability $\gamma$ and otherwise moves to the absorbing state.
Consequently, the probability of remaining in the original dynamics up to time
$t$ is $\gamma^{t-1}$. Conditional on this event, $(s_t,a_t)$ has the same
distribution as under $p_{\pi_\theta}$. Since $f(s_f,a_f)=0$, it follows that
\[
    \E_{\tau\sim\widetilde p_{\pi_\theta}}
    \left[f(s_t,a_t)\right]
    =
    \gamma^{t-1}
    \E_{\tau\sim p_{\pi_\theta}}
    \left[f(s_t,a_t)\right].
\]
Summing over $t\geq 1$ and using the assumed integrability gives
\[
    \E_{\tau\sim\widetilde p_{\pi_\theta}}
    \left[\sum_{t=1}^{\infty} f(s_t,a_t)\right]
    =
    \E_{\tau\sim p_{\pi_\theta}}
    \left[\sum_{t=1}^{\infty}\gamma^{t-1}f(s_t,a_t)\right],
\]
which proves the claim.
\end{proof}

\begin{proof}[Proof of \Cref{prop:irl_problem_formulation}]
Under the absorbing-state construction,
\[
    \log \widetilde p_{\pi_\theta}(\tau)
    =
    \log\widetilde\rho(\tau)
    +
    \sum_{t=1}^{\infty}\log\pi_\theta(a_t\mid s_t),
\]
where the terms after absorption vanish because
$\pi_\theta(a_f\mid s_f)=1$. Moreover,
\[
    \log\widetilde p_\phi(\tau)
    =
    \log\widetilde\rho(\tau)
    +\frac{1}{\alpha}\sum_{t=1}^{\infty}r_\phi(s_t,a_t)
    -\log Z_\phi.
\]
Therefore,
\[
\begin{aligned}
    \alpha\KL\!\left(
        \widetilde p_{\pi_\theta}\,\|\,\widetilde p_\phi
    \right)
    & =
    \E_{\tau\sim\widetilde p_{\pi_\theta}}
    \left[
        \sum_{t=1}^{\infty}
        \left(
            \alpha\log\pi_\theta(a_t\mid s_t)
            -r_\phi(s_t,a_t)
        \right)
    \right]
    +\alpha\log Z_\phi.
\end{aligned}
\]
The normalization term is independent of $\theta$. Applying
\Cref{lem:absorbing-discounting-main} to the expectation shows that minimizing
the inner KL divergence is equivalent to minimizing
\[
    \E_{\tau\sim p_{\pi_\theta}}
    \left[
        \sum_{t=1}^{\infty}\gamma^{t-1}
        \left(
            \alpha\log\pi_\theta(a_t\mid s_t)
            -r_\phi(s_t,a_t)
        \right)
    \right],
\]
which is $\Linner(\theta,\phi)$ in \Eqref{eq:loss_irl_def}.

For the outer problem, terms involving only the expert distribution and the
environment dynamics are independent of $\theta$, so
\[
    \KL\!\left(
        \widetilde p_{\mathrm{expert}}
        \,\|\,
        \widetilde p_{\pi_\theta}
    \right)
    =
    C_{\mathrm{expert}}
    -
    \E_{\tau\sim\widetilde p_{\mathrm{expert}}}
    \left[
        \sum_{t=1}^{\infty}\log\pi_\theta(a_t\mid s_t)
    \right],
\]
where $C_{\mathrm{expert}}$ does not depend on $\theta$. A second application of
\Cref{lem:absorbing-discounting-main} gives $\Louter(\theta)$ in
\Eqref{eq:loss_irl_def}. Hence the bilevel KL problem in
\Eqref{eq:bilevel_irl_def} and the discounted bilevel problem in
\Eqref{eq:loss_no_kl_def} have the same minimizers.
\end{proof}

\begin{proposition}[Discounted outer-objective gradient]
\label{prop:discounted-outer-gradient}
The discounted outer objective satisfies
\[
    \nabla_\theta\Louter(\theta)
    =
    -
    \E_{\tau\sim p_{\mathrm{expert}}}
    \left[
        \sum_{t=1}^{\infty}
        \gamma^{t-1}
        \nabla_\theta\log\pi_\theta(a_t\mid s_t)
    \right].
\]
\end{proposition}

\begin{proof}
The expert trajectory distribution does not depend on $\theta$. Therefore,
using the regularity assumptions stated at the beginning of this appendix,
\[
\begin{aligned}
    \nabla_\theta\Louter(\theta)
    & =
    -\nabla_\theta
    \E_{\tau\sim p_{\mathrm{expert}}}
    \left[
        \sum_{t=1}^{\infty}
        \gamma^{t-1}\log\pi_\theta(a_t\mid s_t)
    \right] \\
    & =
    -\E_{\tau\sim p_{\mathrm{expert}}}
    \left[
        \sum_{t=1}^{\infty}
        \gamma^{t-1}
        \nabla_\theta\log\pi_\theta(a_t\mid s_t)
    \right].
\end{aligned}
\]
\end{proof}

\begin{proposition}[Discounted stepwise form of the trajectory Fisher]
\label{prop:discounted-stepwise-trajectory-fisher}
Under the assumptions of \Cref{prop:inner-hessian-fisher-main},
\[
    \nabla_\theta^2\Linner(\theta^\star(\phi),\phi)
    =
    \alpha
    \left.
    \E_{\tau\sim p_{\pi_\theta}}
    \left[
        \sum_{t=1}^{\infty}
        \gamma^{t-1}
        g_t(\tau)g_t(\tau)^\top
    \right]
    \right|_{\theta=\theta^\star(\phi)},
\]
where
\[
    g_t(\tau)
    \coloneqq
    \nabla_\theta\log\pi_\theta(a_t\mid s_t).
\]

In particular,
$\nabla_\theta^2\Linner(\theta^\star(\phi),\phi)
=\alpha\Fisher_{\theta^\star(\phi)}$.
\end{proposition}

\begin{proof}
By \Cref{lem:absorbing-discounting-main}, the discounted inner objective can equivalently be written under the absorbing-state trajectory distribution as
\[
    \Linner(\theta,\phi)
    =
    \E_{\tau\sim\widetilde p_{\pi_\theta}}
    \left[
        \ell_{\theta,\phi}(\tau)
    \right],
\]
where
\[
    \ell_{\theta,\phi}(\tau)
    :=
    \sum_{t=1}^{\infty}
    \left(
        \alpha\log\pi_\theta(a_t\mid s_t)
        -
        r_\phi(s_t,a_t)
    \right).
\]
Define the trajectory score and its derivative by
\[
    S_\theta(\tau)
    :=
    \nabla_\theta
    \log\widetilde p_{\pi_\theta}(\tau),
    \qquad
    H_\theta(\tau)
    :=
    \nabla_\theta^2
    \log\widetilde p_{\pi_\theta}(\tau).
\]
Since the transition kernel is independent of \(\theta\),
\[
    S_\theta(\tau)
    =
    \sum_{t=1}^{\infty}
    \nabla_\theta\log\pi_\theta(a_t\mid s_t)
    =
    \sum_{t=1}^{\infty}g_t(\tau).
\]
After absorption, all terms vanish because
\(\pi_\theta(a_f\mid s_f)=1\), so these sums contain only finitely many
nonzero terms almost surely.

Using the score-function identity and
\(\nabla_\theta\ell_{\theta,\phi}(\tau)=\alpha S_\theta(\tau)\), we obtain
\[
    \nabla_\theta\Linner(\theta,\phi)
    =
    \E_{\tau\sim\widetilde p_{\pi_\theta}}
    \left[
        \bigl(\ell_{\theta,\phi}(\tau)+\alpha\bigr)
        S_\theta(\tau)
    \right].
\]
Differentiating once more gives
\[
\begin{aligned}
    \nabla_\theta^2\Linner(\theta,\phi)
    =
    \E_{\tau\sim\widetilde p_{\pi_\theta}}
    \Big[
        \bigl(\ell_{\theta,\phi}(\tau)+2\alpha\bigr)
        S_\theta(\tau)S_\theta(\tau)^\top +
        \bigl(\ell_{\theta,\phi}(\tau)+\alpha\bigr)
        H_\theta(\tau)
    \Big].
\end{aligned}
\]
Equivalently,
\[
\begin{aligned}
    \nabla_\theta^2\Linner(\theta,\phi)
    =
    \E_{\tau\sim\widetilde p_{\pi_\theta}}
    \Big[
        \alpha S_\theta(\tau)S_\theta(\tau)^\top +
        \bigl(\ell_{\theta,\phi}(\tau)+\alpha\bigr)
        \bigl(
            S_\theta(\tau)S_\theta(\tau)^\top
            +H_\theta(\tau)
        \bigr)
    \Big].
\end{aligned}
\]

The trajectory score satisfies
\[
    \E_{\tau\sim\widetilde p_{\pi_\theta}}
    \left[
        S_\theta(\tau)S_\theta(\tau)^\top
        +H_\theta(\tau)
    \right]
    =0,
\]
which follows by differentiating
\(\int\widetilde p_{\pi_\theta}(\tau)\,d\tau=1\) twice with respect to
\(\theta\).

We now evaluate the Hessian at \(\theta=\theta^\star(\phi)\). The two
trajectory distributions satisfy
\[
    \log\widetilde p_{\pi_\theta}(\tau)
    -
    \log\widetilde p_\phi(\tau)
    =
    \frac{1}{\alpha}\ell_{\theta,\phi}(\tau)
    +
    \log Z_\phi.
\]
Under exact realization,
\[
    \widetilde p_{\pi_{\theta^\star(\phi)}}
    =
    \widetilde p_\phi \eqsp,
\]
and hence
\[
    \ell_{\theta^\star(\phi),\phi}(\tau)
    =
    -\alpha\log Z_\phi
    \qquad\text{a.s.}
\]
Thus,
\(\ell_{\theta^\star(\phi),\phi}(\tau)+\alpha\) is constant with
respect to \(\tau\). Using the trajectory score identity above,
\[
\begin{aligned}
    &\E_{\tau\sim\widetilde p_{\pi_{\theta^\star(\phi)}}}
    \left[
        \bigl(\ell_{\theta^\star(\phi),\phi}(\tau)+\alpha\bigr)
        \bigl(
            S_{\theta^\star(\phi)}S_{\theta^\star(\phi)}^\top
            +H_{\theta^\star(\phi)}
        \bigr)
    \right]
    =0.
\end{aligned}
\]
Consequently,
\[
    \nabla_\theta^2\Linner(\theta^\star(\phi),\phi)
    =
    \alpha
    \left.
    \E_{\tau\sim\widetilde p_{\pi_\theta}}
    \left[
        S_\theta(\tau)S_\theta(\tau)^\top
    \right]
    \right|_{\theta=\theta^\star(\phi)}.
\]

It remains to obtain the stepwise form. Expanding the trajectory score gives
\[
    S_\theta(\tau)S_\theta(\tau)^\top
    =
    \sum_{t=1}^{\infty}g_tg_t^\top
    +
    \sum_{t\neq k}g_tg_k^\top.
\]
For \(t<k\), the conditional score identity yields
\[
    \E\!\left[
        g_k(\tau)
        \mid
        s_1,a_1,\ldots,s_{k-1},a_{k-1},s_k
    \right]
    =0.
\]
Since \(g_t\) is measurable with respect to the conditioning variables,
\[
    \E[g_tg_k^\top]=0.
\]
The case \(k<t\) follows by transposition. Therefore,
\[
    \E_{\tau\sim\widetilde p_{\pi_\theta}}
    \left[
        S_\theta(\tau)S_\theta(\tau)^\top
    \right]
    =
    \E_{\tau\sim\widetilde p_{\pi_\theta}}
    \left[
        \sum_{t=1}^{\infty}g_t(\tau)g_t(\tau)^\top
    \right].
\]

Finally, applying \Cref{lem:absorbing-discounting-main} componentwise gives
\[
    \E_{\tau\sim\widetilde p_{\pi_\theta}}
    \left[
        \sum_{t=1}^{\infty}g_tg_t^\top
    \right]
    =
    \E_{\tau\sim p_{\pi_\theta}}
    \left[
        \sum_{t=1}^{\infty}
        \gamma^{t-1}g_tg_t^\top
    \right].
\]
Combining the preceding identities and evaluating at
\(\theta=\theta^\star(\phi)\) yields
\[
    \nabla_\theta^2\Linner(\theta^\star(\phi),\phi)
    =
    \alpha
    \left.
    \E_{\tau\sim p_{\pi_\theta}}
    \left[
        \sum_{t=1}^{\infty}
        \gamma^{t-1}
        g_t(\tau)g_t(\tau)^\top
    \right]
    \right|_{\theta=\theta^\star(\phi)}
    =
    \alpha\Fisher_{\theta^\star(\phi)}.
\]
\end{proof}

\begin{proposition}[Discounted prefix-score form of the inner cross derivative]
\label{prop:discounted-prefix-score-inner-cross-derivative}
The inner cross derivative satisfies
\[
    \frac{\partial^2\Linner}{\partial\phi\,\partial\theta}
    =
    -
    \E_{\tau\sim p_{\pi_\theta}}
    \left[
        \sum_{t=1}^{\infty}
        \gamma^{t-1}
        \left(
            \frac{\partial r_\phi(s_t,a_t)}{\partial\phi}
        \right)^\top
        \left(
            \sum_{k=1}^{t}
            \frac{\partial}{\partial\theta}
            \log\pi_\theta(a_k\mid s_k)
        \right)
    \right].
\]
\end{proposition}

\begin{proof}
Differentiating the discounted inner objective first with respect to $\phi$
gives
\[
    \frac{\partial\Linner}{\partial\phi}
    =
    -
    \E_{\tau\sim p_{\pi_\theta}}
    \left[
        \sum_{t=1}^{\infty}
        \gamma^{t-1}
        \frac{\partial r_\phi(s_t,a_t)}{\partial\phi}
    \right].
\]
For a quantity determined by the trajectory up to time $t$, differentiation of
its expectation with respect to $\theta$ introduces only the policy scores up to
that time. Hence, by the score-function identity and causality,
\[
    \frac{\partial}{\partial\theta}
    \E_{\tau\sim p_{\pi_\theta}}
    \left[
        \frac{\partial r_\phi(s_t,a_t)}{\partial\phi}
    \right]
    =
    \E_{\tau\sim p_{\pi_\theta}}
    \left[
        \left(
            \frac{\partial r_\phi(s_t,a_t)}{\partial\phi}
        \right)^\top
        \left(
            \sum_{k=1}^{t}
            \frac{\partial}{\partial\theta}
            \log\pi_\theta(a_k\mid s_k)
        \right)
    \right].
\]
Interchanging differentiation with the absolutely convergent discounted series
then gives the claimed expression.
\end{proof}

\begin{corollary}[JVP form of the inner cross-derivative-vector product]
\label{cor:jvp-inner-cross-derivative}
Let $\vv\in\R^{d_\theta}$ and define the scalar sequence
\[
    q_t(\tau)
    \coloneqq
    \left(
        \frac{\partial}{\partial\theta}
        \log\pi_\theta(a_t\mid s_t)
    \right)\vv,
    \qquad
    b_t(\tau)
    \coloneqq
    \sum_{k=1}^{t}q_k(\tau).
\]
Then
\[
    \frac{\partial^2\Linner}{\partial\phi\,\partial\theta}\vv
    =
    -
    \E_{\tau\sim p_{\pi_\theta}}
    \left[
        \nabla_\phi
        \left(
            \sum_{t=1}^{\infty}
            \gamma^{t-1}r_\phi(s_t,a_t)b_t(\tau)
        \right)
    \right].
\]
\end{corollary}

\begin{proof}
Multiplying the identity in
\Cref{prop:discounted-prefix-score-inner-cross-derivative} by $\vv$ gives
\[
    \frac{\partial^2\Linner}{\partial\phi\,\partial\theta}\vv
    =
    -
    \E_{\tau\sim p_{\pi_\theta}}
    \left[
        \sum_{t=1}^{\infty}
        \gamma^{t-1}
        \left(
            \frac{\partial r_\phi(s_t,a_t)}{\partial\phi}
        \right)^\top
        b_t(\tau)
    \right].
\]
The assumed absolute convergence justifies the exchange of the matrix-vector
product and the series. Since $b_t(\tau)$ is independent of $\phi$, it can be
treated as constant when differentiating with respect to $\phi$, which gives
the stated JVP representation.
\end{proof}

\newpage
\section{Algorithmic Details}
\label{app:algorithms}


\SCFD

\newpage
\section{Additional Computational Results}
\label{app:additional-computational-results}

\begin{table}[!htbp]
    \centering
    \caption{
        Wall-clock time per outer iteration across the considered
        environments. For runs executed under the fixed 24-hour
        computational budget, the iteration time is computed using
        the number of completed outer iterations.
    }
    \label{tab:outer-iteration-time}
    \small
    \setlength{\tabcolsep}{6pt}
    \renewcommand{\arraystretch}{1.1}

    \begin{tabular}{l c c}
        \toprule
        Method & Time / outer iteration (s) & Runs \\
        \midrule

        \multicolumn{3}{c}{\textbf{CartPole}} \\
        \midrule
        ML-IRL
            & $510.20 \pm 120.77$
            & 3 \\
        Explicit Fisher
            & $1464.41$
            & 1 \\
        Fisher with Sketching ($m=8$)
            & $1132.08 \pm 130.33$
            & 3 \\

        \midrule\midrule

        \multicolumn{3}{c}{\textbf{LQR}} \\
        \midrule
        ML-IRL
            & $53.61 \pm 6.57$
            & 3 \\
        Explicit Fisher
            & $1037.76 \pm 120.69$
            & 8 \\
        Fisher with Sketching ($m=32$)
            & $995.66 \pm 91.16$
            & 3 \\
        Fisher with Sketching ($m=64$)
            & $1171.03 \pm 188.30$
            & 4 \\
        Fisher with Sketching ($m=128$)
            & $1042.90 \pm 160.41$
            & 2 \\
        Fisher with Sketching ($m=256$)
            & $1316.88$
            & 1 \\

        \bottomrule
    \end{tabular}
\end{table}

\begin{table}[!htbp]
    \centering
    \caption{
        Peak memory usage across the considered environments.
        Peak memory is measured as the maximum resident set size (RSS)
        observed during training.
    }
    \label{tab:peak-memory}
    \small
    \setlength{\tabcolsep}{6pt}
    \renewcommand{\arraystretch}{1.1}

    \begin{tabular}{l c c}
        \toprule
        Method & Peak RSS (MB) & Runs \\
        \midrule

        \multicolumn{3}{c}{\textbf{CartPole}} \\
        \midrule
        ML-IRL
            & $645.57 \pm 1.56$
            & 3 \\
        Explicit Fisher
            & $644.57$
            & 1 \\
        Fisher with Sketching ($m=8$)
            & $662.14 \pm 30.00$
            & 4 \\

        \midrule\midrule

        \multicolumn{3}{c}{\textbf{LQR}} \\
        \midrule
        ML-IRL
            & $779.83 \pm 120.03$
            & 3 \\
        Explicit Fisher
            & $929.61 \pm 198.11$
            & 10 \\
        Fisher with Sketching ($m=32$)
            & $725.48 \pm 99.84$
            & 3 \\
        Fisher with Sketching ($m=64$)
            & $710.89 \pm 102.31$
            & 4 \\
        Fisher with Sketching ($m=128$)
            & $767.64 \pm 122.08$
            & 2 \\
        Fisher with Sketching ($m=256$)
            & $844.58$
            & 1 \\

        \bottomrule
    \end{tabular}
\end{table}

\end{document}